\documentclass[11pt]{article} % For LaTeX2e

\usepackage{amsthm,amsmath}

\usepackage[dvips]{graphicx}
\usepackage{amsfonts}
\usepackage{amssymb}
\usepackage{mathrsfs}

\usepackage{wrapfig}

\newcommand{\eps}{\varepsilon}

\newcommand{\R}{{\mathbb{R}}}

\newcommand{\cH}{\mathcal{H}}
\newcommand{\cX}{\mathcal{X}}
\newcommand{\cY}{\mathcal{Y}}

\newcommand{\Hx}{{\cH_\cX}}
\newcommand{\Hy}{{\cH_\cY}}

\newcommand{\hC}{\widehat{C}^{(n)}}

\newcommand{\la}{\langle}
\newcommand{\ra}{\rangle}

\newcommand{\indep}{{\bot\negthickspace\negthickspace\bot}}
\newcommand{\eq}[1]{Eq.~(\ref{#1})}
\newtheorem{thm}{Theorem}

\title{Gradient-based kernel dimension reduction for supervised learning}

\author{
Kenji Fukumizu\footnote{The Institute of Statistical Mathematics, 
10-3 Midori-cho, Tachikawa, Tokyo 190-8562 Japan} and Chenlei Leng\footnote{Department of Statistics and Applied Probability,
National University of Singapore,
6 Science Drive 2, Singapore, 117546
}}

\begin{document}

\maketitle

\begin{abstract}
This paper proposes a novel kernel approach to linear dimension reduction for supervised learning.  The purpose of the dimension reduction is to find directions in the input space to explain the output as effectively as possible.  The proposed method uses an estimator for the gradient of regression function, based on the covariance operators on reproducing kernel Hilbert spaces.  In comparison with other existing methods, the proposed one has wide applicability without strong assumptions on the distributions or the type of variables, and uses computationally simple eigendecomposition.  Experimental results show that the proposed method successfully finds the effective directions with efficient computation.
\end{abstract}

\section{Introduction}

Dimension reduction is involved in most of modern data analysis, in which high dimensional data must be often handled.  The purpose of dimension reduction is multifold: preprocessing for another data analysis, aiming at less expensive computation in later processing, or construction of readable low dimensional expressions.  There are two categories of dimension reduction: unsupervised methods such as PCA, and supervised methods such as Fisher discriminant analysis (FDA).  This paper focuses on dimension reduction for supervised learning.

Let $(X,Y)$ be a random vector such that $X$ takes values in $\R^m$. The domain of $Y$ can be arbitrary, either continuous, discrete, or structured.  Supervised learning concerns how $Y$ is explained by $X$.  The purpose of dimension reduction in this setting is to find such features of $X$ that explain $Y$ as effectively as possible.
This paper focuses linear dimension reduction, in which linear combinations of the components of $X$ are used to make effective features.  Although there are many methods for extracting nonlinear features including kernel methods, this paper confines its attentions on linear features for the following reasons: (i) nonlinear feature extraction such as kernel method depends strongly on the choice of the nonlinearity (see Sec.~\ref{sec:real_data}, Wine data, for example). Linear methods are more stable.  (ii) we can apply some nonlinear transform $\phi(X)$ of $X$ so that linear combinations of $\phi(X)$ give effective features of $X$, once a linear dimension reduction method is established.

Beyond the classical approaches such as FDA and CCA, the modern approach to this linear dimension reduction is based on the formulation by conditional independence.  More precisely, we assume
\begin{equation}\label{eq:edr}
    p(Y|X) = \tilde{p}(Y|B^TX) \qquad \text{or equivalently} \qquad Y\indep X\,|\, B^T X
\end{equation}
for the distribution, where $B$ is a projection matrix ($B^TB = I_d$) onto a $d$-dimensional subspace ($d< m$) in $\R^m$, and wish to estimate $B$.  The subspace spanned by the column vectors of $B$ is called the {\em effective direction for regression}, or {\em EDR space} \cite{KerChauLi92}.  We consider methods of estimating $B$ without specific parametric models for $p(y|x)$, unlike the model-based approach such as \cite{Rish_etal_ICML2008}

The first method that aims at finding the EDR space is the {\em sliced inverse regression} (SIR, \cite{KerChauLi91}), which employs the fact that the inverse regression $E[X|Y]$ lies in the EDR space under some assumptions.  Many methods have been proposed in this vein of inverse regression (\cite{CookWeisberg91,Li_etal_2005}), which use some statistic in each slice of $Y$.  While many inverse regression methods are computationally simple, they often need some strong assumptions on the distribution of $X$ such as elliptic symmetry, and slice-based methods are not effective for classification, where the number of slices is at most that of classes.  Another interesting approach is the minimum average variance estimation (MAVE \cite{Xia_etal_2002MAVE}), in which the conditional variance of the regression in the direction of $B^TX$, $E[ (Y - E[Y|B^TX])^2 | B^T X]$, is minimized with the conditional variance estimated by the local linear kernel smoothing method.  The kernel smoothing method requires, however, careful choice of bandwidth parameter, and it is usually difficult to apply if the dimensionality is very high.

The most relevant to this paper is the methods that use the gradient of regressor $\varphi(x)=E[Y|X=x]$ \cite{Samarov93,Hristache01}.  As explained in Sec.~\ref{sec:grad}, under \eq{eq:edr} the gradient of $\varphi(x)$ is contained in the EDR space.  One can estimate the space by nonparametric estimation of the gradient.  There are some limitations in this method, however: the nonparametric estimation of the gradient in high-dimensional spaces is challenging, and the gradient is not estimable if some symmetry holds in the system.
%This paper proposes a novel kernel approach to overcome these limitations of the existing method.

A kernel method for dimension reduction has been proposed to overcome various limitations of existing methods. The kernel dimension reduction (KDR, \cite{Fukumizu04_jmlr,Fukumizu_etal09_KDR,Wang_etalNIPS2010}) uses the kernel method to characterize the conditional independence relation in \eq{eq:edr}.  While KDR is a general method applicable to a wide class of problems without requiring any strong assumptions on the distributions or types of $X$ or $Y$, the optimization needed for the estimation is computationally a problem: the objective function is non-convex, and the gradient descent method demands many inversions of Gram matrices, which prohibits applications to very high-dimensional or large data.

We propose a novel kernel method for dimension reduction using the gradient-based approach, but unlike the existing ones \cite{Samarov93,Hristache01}, the gradient is estimated by the covariance operators with positive definite kernels, which is based on the recent development in the kernel method \cite{Fukumizu_etal09_KDR,Song_etal_ICML2009}.  It solves the problems of existing methods: by virtue of the kernel method the response $Y$ can be of arbitrary type, and the kernel estimation of the gradient is stable without careful decrease of bandwidth.  It solves also the problem of KDR: the estimator by an eigenproblem needs no numerical optimization.  The method is thus applicable to large and high-dimensional data, as we demonstrate experimentally.
%The applicability of the method is thus even wider than KDR.

\section{Gradient-based kernel dimension reduction}

In this paper, the range of an operator $A$ is denoted by $\mathcal{R}(A)$.

\subsection{Gradient of a regression function and dimension reduction}
\label{sec:grad}

We first review the basic idea of the gradient-based method for dimension reduction in supervised learning, which has been used in \cite{Samarov93,Hristache01}.
Suppose $Y$ is a real-valued random variable such that the regression function $E[Y|X=x]$ is differentiable w.r.t.~$x$.  If the assumption \eq{eq:edr} holds, we have
\begin{align*}
    \frac{\partial}{\partial x} E[Y|X=x] &
    = \frac{\partial}{\partial x} \int yp(y|x)dy  = \int y\frac{\partial\tilde{p}(y|B^Tx)}{\partial x}dy = B \int y\left.\frac{\partial\tilde{p}(y|z)}{\partial z}\right|_{z=B^Tx}dy,
\end{align*}
which implies that the gradient $\frac{\partial}{\partial x} E[Y|X=x]$ at any $x$ is contained in the EDR space.  Based on this fact, the average derivative estimates (ADE, \cite{Samarov93}) has been proposed to use the average of the gradients for estimating $B$.   In the more recent method \cite{Hristache01}, assuming that $Y$ is one-dimensional continuous variable, a standard local linear least squares with a smoothing kernel (not necessarily positive definite kernel) \cite{FanGijbels1996} is used for estimating the gradient, and the dimensionality of the projection is iteratively reduced to the desired one. Since the gradient estimation for high-dimensional data is difficult in general, the iterative reduction is expected to give a more accurate estimation.  We call the method in \cite{Hristache01} iterative average derivative estimates (IADE).

\subsection{Kernel method for conditional expectation}

It has been recently revealed that the apparatus of positive definite kernels or reproducing kernel Hilbert space (RKHS) can be applied to estimate the regression function or conditional expectation with covariance operators on RKHS \cite{Fukumizu04_jmlr,Fukumizu_etal09_KDR,Song_etal_ICML2009}, which we briefly review below.
For a set $\Omega$, a ($\R$-valued) positive definite kernel $k$ on $\Omega$ is a symmetric kernel $k:\Omega\times\Omega\to\R$ such that $\sum_{i,j=1}^n c_i c_j k(x_i,x_j)\geq 0$ for any $x_1,\dots,x_n$ in $\Omega$ and $c_1,\ldots,c_n\in\R$.  It is known that a positive definite kernel on $\Omega$ uniquely defines a Hilbert space $\cH$ consisting of functions on $\Omega$ such that (i) $k(\cdot,x)$ is in $\cH$, (ii) the linear hull of $\{ k(\cdot,x)\mid x\in \Omega\}$ is dense in $\cH$, and (iii) for any $x\in \Omega$ and $f\in\cH$, $\la f,k(\cdot,x)\ra_\cH = f(x)$ (reproducing property), where $\la\cdot,\cdot\ra_\cH$ is the inner product of $\cH$.
The Hilbert space $\cH$ is called the {\it reproducing kernel Hilbert space} (RKHS) associated with $k$.

Let $(\cX,\mathcal{B}_\cX,\mu_{\cX})$ and
$(\cY,\mathcal{B}_\cY,\mu_{\cY})$ be measure spaces, and $(X,Y)$ be
a random variable on $\cX\times\cY$ with probability $P$. We assume that the probability density function (p.d.f.)~$p(x,y)$ and the
conditional p.d.f.~$p(y|x)$ always exist.  Also, we always assume that a positive definite kernel is measurable and bounded: the boundedness means $\sup_{x\in\Omega}k(x,x)< \infty$.

Let $k_\cX$ and $k_\cY$ be positive definite kernels on $\cX$ and $\cY$, respectively, with respective RKHS $\Hx$ and $\Hy$.
The (uncentered) {\em covariance operator} $C_{YX}:
\Hx\to\Hy$ is defined by the equation
\begin{equation}\label{eq:cov_op}
    \la g, C_{YX} f\ra_\Hy = E[f(X)g(Y)] = E\bigl[ \la f, \Phi_\cX(X)\ra_\Hx \la\Phi_\cY(Y),g\ra_\Hy \bigr]
\end{equation}
for all $f\in\Hx,g\in\Hy$, where $\Phi_\cX(x)=k_\cX(\cdot,x)$ and $\Phi_\cY(y)=k_\cY(\cdot,y)$.  Similarly,
$C_{XX}$ denotes the operator on $\Hx$ that satisfies $\la f_2, C_{XX}f_1\ra = E[f_2(X)f_1(X)]$ for any $f_1,f_2\in\Hx$.  These definitions are straightforward extensions of the ordinary covariance matrices, if we consider the covariance of the random vectors $\Phi_\cX(X)$ and $\Phi_\cY(Y)$ on RKHS.

By setting $g=k_\cY(\cdot,y)$ in \eq{eq:cov_op}, the reproducing property derives
\[
(C_{YX}f )(y) = \int k_\cY(y,\tilde{y})f(\tilde{x})dP(\tilde{x},\tilde{y}),\quad
(C_{XX}f )(x) = \int k_\cX(x,\tilde{x})f(\tilde{x})dP_X(\tilde{x}),
\]
which shows the explicit expressions of $C_{YX}$ and $C_{XX}$ as integral operators.

An advantage of the kernel method is that estimation with finite data is straightforward.  Given i.i.d.~sample $(X_1,Y_1),\ldots,(X_n,Y_n)$ with law $P$, the covariance operator is estimated by
\begin{equation}\label{eq:emp_covop}
    \hC_{YX}f = \frac{1}{n} \sum_{i=1}^n k_\cY(\cdot,Y_i)\la  k_\cX(\cdot,X_i), f\ra_\Hx = \frac{1}{n} \sum_{i=1}^n f(X_i) k_\cY(\cdot,Y_i).
\end{equation}
The estimator $\hC_{XX}$ is given similarly.
It is known that these estimators are $\sqrt{n}$-consistent in Hilbert-Schmidt norm \cite{Gretton_etal05AISTATS}.

The fundamental result in discussing conditional probabilities with kernels is the following fact.
\begin{thm}[\cite{Fukumizu04_jmlr}]\label{thm:cond_mean}
If $E[g(Y)|X=\cdot]\in\Hx$ holds for $g\in\Hy$, then
\[
    C_{XX} E[g(Y)|X=\cdot]=C_{XY} g.
\]
\end{thm}
If $C_{XX}$ is injective\footnote{Noting $\la C_{XX}f, f\ra=E[f(X)^2]$, it is easy to see that $C_{XX}$ is injective, if $k_\cX$ is a continuous kernel on a topological space $\cX$, and $P_X$ is a
Borel probability measure such that $P(U)>0$ for any open set $U$ in $\cX$.}, the above relation can be expressed as
\begin{equation}\label{eq:cond_basic}
    E[g(Y)|X=\cdot]={C_{XX}}^{-1}C_{XY} g.
\end{equation}

The assumption $E[g(Y)|X=\cdot]\in \Hx$ may not hold in general; we can easily make counterexamples with Gaussian kernel and Gaussian distributions.  We can nonetheless obtain an empirical estimator
based on \eq{eq:cond_basic}, namely,
\[
    (\hC_{XX}+\eps_n I)^{-1}\hC_{XY}g,
\]
where $\eps_n$ is a regularization coefficient in Thikonov-type regularization.  As we discuss in Appendix, we can in fact prove rigorously that
this estimator converges to $E[g(Y)|X=\cdot]$.
%\begin{thm}
%If $\eps_n$ satisfies $\eps_n\to 0$ and $n \eps_n^3 \to \infty$ as $n\to\infty$, then for arbitrary $g\in \Hy$,
%\[
%\bigl\| \bigl(\widehat{C}_{XX}^{(n)}+\eps_n I\bigr)^{-1}\hC_{XY} g
% - E[g(Y)|X]\bigr\|_{L^2(P_X)} \to 0
%\]
%in probability as $n\to \infty$.
%\end{thm}

%Suppose $k_\cX(\cdot,x) \in \mathcal{R}(C_{XX})$\footnote{This assumption may not hold in general.  But, the motivation here is to construct an empirical estimator from the population expression which may need strong assumption.}, and $E[g(Y)|X=\cdot]\in\Hx$ for
%any $g\in \Hy$.  If further $C_{XX}$ is injective, \eq{eq:cond_basic} means
%\begin{align*}
%    \la g, E[k_\cY(\cdot,Y)|X=x]\ra_\Hy & = \la E[g(Y)|X=\cdot], k_\cX(\cdot,x)\ra_\Hx  = \la C_{XX}^{-1}C_{XY}g, k_\cX(\cdot,x)\ra_\Hx \\
%    & = \la g, C_{YX}C_{XX}^{-1}k_\cX(\cdot,x)\ra_\Hx,
%\end{align*}
%from which we have
%\begin{equation}\label{eq:cond_exp}
%E[k_\cY(\cdot,Y)|X=x] = C_{YX}C_{XX}^{-1} k_\cX(\cdot,x).
%\end{equation}
%As discussed in \cite{Song_etal_ICML2009}, the operator
%$C_{YX}C_{XX}^{-1}$ can be regarded as the kernel expression of
%the conditional probability $p(y|x)$.

To apply the above kernel expressions to the method discussed in Sec.~\ref{sec:grad}, we need a way of taking the derivative of a function.
It is known ({\em e.g.}, \cite{SteinwartChristmann2008} Sec.~4.3) that if a positive definite kernel $k(x,y)$ on an open set in Euclidean space is continuously differentiable with respect to $x$ and $y$, every $f$ in the corresponding RKHS is continuously differentiable. If further $\frac{\partial }{\partial x}k(\cdot,x)\in\Hx$, we have
\begin{equation}\label{eq:kernel_deriv}
\frac{\partial f}{\partial x} = \left\la f, \frac{\partial }{\partial x}k(\cdot,x)\right\ra_\Hx.
\end{equation}
Namely, the derivative of any function in that RKHS can be computed in the form of the inner product.
This property combined with the above kernel estimator of $E[g(Y)|X=x]$ provides a method for dimension reduction.

\subsection{Gradient-based kernel method for dimension reduction}

\subsubsection{Algorithm}
\label{sec:alg}

%Based on the discussions in previous sections, we will propose a new kernel method of dimension reduction for supervised learning.
Assume that $\cX=\R^m$, $C_{XX}$ is injective, $k_\cX(x,\tilde{x})$ is continuously differentiable, $E[g(Y)|X=x]\in \Hx$ for any $g\in\Hy$, and $\frac{\partial }{\partial x}k_\cX(\cdot,x)\in\mathcal{R}(C_{XX})$.
%where $\mathcal{R}(\cdot)$ is the range of an operator.
It follows from Eqs.~(\ref{eq:cond_basic}) and (\ref{eq:kernel_deriv}) that
\begin{equation}
    \frac{\partial}{\partial x}E[g(Y)|X=x]
      = \Bigl\la C_{XX}^{-1} C_{XY}g , \frac{\partial k_\cX(\cdot,x)}{\partial x}
    \Bigr\ra
     = \Bigl\la g,  C_{YX}C_{XX}^{-1}\frac{\partial k_\cX(\cdot,x)}{\partial x}\Bigr\ra.
    \label{eq:deriv}
\end{equation}
Define $\Psi:\R^m\to\Hy$, $x\mapsto E[k_\cY(\cdot,Y)|X=x]$.  By plugging $g=k(\cdot,y)$ into \eq{eq:deriv}, we see
\[
    \frac{\partial \Psi(x)}{\partial x} = C_{YX}C_{XX}^{-1} \frac{\partial k_\cX(\cdot,x)}{\partial x}.
\]
On the other hand, from $\Psi(x) = \int k_\cY(\cdot,y)p(y|x)d\mu_y(y)$, the same argument as in Sec.~\ref{sec:grad} shows that $\frac{\partial \Psi(x)}{\partial x}=\Xi(x)B$ with an operator $\Xi(x)$ from $\R^m$ to $\Hy$, where we use a slight abuse of notation by identifying the  operator $\Xi(x)$ with a matrix.  Taking the inner product in $\Hy$, we have
\[
B^T \la \Xi(x), \Xi(x)\ra_\Hy B = \Bigl\la \frac{\partial k_\cX(\cdot,x)}{\partial x}, C_{XX}^{-1}C_{XY} C_{YX}C_{XX}^{-1} \frac{\partial k_\cX(\cdot,x)}{\partial x}\Bigr\ra =:M(x),
\]
which shows that the eigenvectors for non-zero eigenvalues of the $m\times m$ symmetric matrix $M(x)$ are contained in the EDR space.  This fact is the basis of the proposed method.
Note that, in comparison with the conventional gradient-based method described in Sec.~\ref{sec:grad},  this method is interpreted as considering simultaneously various regression functions $E[k_\cY(\tilde{y},Y)|X=x]$ given by all $\tilde{y}\in\cY$.

Given i.i.d.~sample $(X_1,Y_1),\ldots,(X_n,Y_n)$ from the true distribution,
based on the empirical covariance operators \eq{eq:emp_covop} and regularized inversions, the matrix $M(x)$ is estimated by
\begin{align}\label{eq:Mn}
    \widehat{M}_n(x) & = \bigl\la \tfrac{\partial k_\cX(\cdot,x)}{\partial x}, \bigl( \hC_{XX}+\eps_n I\bigr)^{-1}\hC_{XY} \hC_{YX}\bigl( \hC_{XX}+\eps_n I\bigr)^{-1} \tfrac{\partial k_\cX(\cdot,x)}{\partial x}\bigr\ra  \nonumber \\
    & = \nabla{\bf k}_X(x)^T(G_X + n\eps_n I)^{-1}G_Y (G_X + n\eps_n I)^{-1}\nabla{\bf k}_X(x),
\end{align}
where $G_X$ and $G_Y$ are Gram matrices $(k_\cX(X_i,X_j))$ and $(k_\cY(Y_i,Y_j))$, respectively, and $\nabla{\bf k}_X(x) =  (
    \frac{\partial k_\cX(X_1,x)}{\partial x} ,
    \cdots ,
    \frac{\partial k_\cX(X_n,x)}{\partial x})^T \in\R^n$.

As the eigenvectors of $M(x)$ are contained in the EDR space for any $x$, we propose to use the average of $M(X_i)$ over all the data points $X_i$, and define
\[
    \tilde{M}_n=\tfrac{1}{n}{\textstyle \sum_{i=1}^n} \widehat{M}_n(X_i) = \tfrac{1}{n}{\textstyle \sum_{i=1}^n} \nabla{\bf k}_X(X_i)^T(G_X + n\eps_n I_n)^{-1}G_Y (G_X + n\eps_n I_n)^{-1}\nabla{\bf k}_X(X_i).
\]
In the case of Gaussian kernel, for example, $\nabla{\bf k}_X(X_i)$ is given by
$(X_i - X_j)\exp( -\frac{1}{2\sigma^2} \|X_i-X_j\|^2 )$,
which is the Hadamard product between the Gram matrix $G_X$ and $(X_i-X_j)_{ij=1}^n$.

The projection matrix $B$ in \eq{eq:edr} is then estimated by the top $d$ eigenvectors of the $m\times m$ symmetric matrix $\tilde{M}_n$.  We call this method {\em gradient-based kernel dimension reduction} (gKDR).
%The algorithm of gKDR is summarized in Figure \ref{fig:alg}.

\subsubsection{Discussions and extensions}
\label{sec:gKDR_discussion}

The proposed gKDR applies to a wide class of problems.  In contrast to many existing methods, the gKDR can handle any type of data for $Y$ including multivariate or structured variables, and make no strong assumptions on the distribution of $X$.  The gKDR method can be applied to classification and continuous output exactly in the same manner.

The previous gradient-based methods ADE and IADE have an obvious weakness.  Suppose $Y$ is one-dimensional and $Y=\varphi(B^TX)+Z$, where $Z$ is a zero-mean noise.  If $E[\varphi'(B^TX)]=0$, the subspace spanned by $B$ cannot be estimated.  This condition holds if $\varphi$ and the distribution of $X$ satisfy some symmetry.  These methods in general find only a subspace of the EDR space. In contrast, the gKDR approach incorporates various functions $k_\cY(\tilde{y},\cdot)$ for $\varphi$, as discussed in Sec.~\ref{sec:alg}, and thus this weakness may be avoided.

As in all kernel methods, the results of gKDR depend on the choice of kernels, though the linear features are less sensitive to the choice than nonlinear features.  We use the  cross-validation (CV) for choosing kernels and parameters, combined with some regression or classification method.  In this paper, the k-nearest neighbor (kNN) regression / classification is used in CV for its simplicity: for each candidate of a kernel or parameter, we compute the CV error by the kNN method with the input data projected on the subspace given by gKDR, and choose the one that gives the least error.

The time complexity of the matrix inversions and the eigendecomposition required for gKDR are $O(n^3)$, which is prohibitive for large data sets.  We can apply, however, low-rank approximation of Gram matrices, such as incomplete Cholesky decomposition \cite{FineScheinberg2001}, which is a standard method for reducing time complexity in kernel methods.  The space complexity may be also a problem of gKDR, since $(\nabla {\bf k}_X(X_i))_{i=1}^n$ has $n^2\times m$ dimension.  In the case of Gaussian kernel, we have a way of reducing the necessary memory by low rank approximation of the Gram matrices.  Note that $\frac{\partial }{\partial x^a}k_X(X_j,x)|_{x=X_i}$ for Gaussian kernel is given by $\frac{1}{\sigma^2}(X_j^a-X_i^a)\exp(-\|X_j-X_i\|^2/(2\sigma^2))$.  Let $G_X\approx R R^T$ and $G_Y\approx H H^T$ be the low rank approximation with $r_x={\rm rk}R, r_y={\rm  rk}H$ ($r_x,r_y < n,m$).  With the notation $F:=(G_X+n\eps_n I_n)^{-1}H$ and $\Theta_i^{as}=\frac{1}{\sigma^2}X_i^a R_{is}$, we have
\[
    \tilde{M}_{n,ab} =  \sum_{i=1}^n\sum_{t=1}^{r_y} \Gamma_{ia}^t \Gamma_{ib}^t
    \qquad (1\leq a,b \leq m),
\]
\[
    \Gamma_{ia}^t = \sum_{j=1}^n\sum_{s=1}^{r_x} \frac{1}{\sigma^2} (X_j^a-X_i^a)R_{js}R_{is}F_{jt} = \sum_{s=1}^{r_x} R_{is}\Bigl(\sum_{j=1}^n \Theta_{j}^{as}F_{jt}\Bigr) - \sum_{s=1}^{r_x} \Theta_{i}^{as}\Bigl( \sum_{j=1}^n R_{js}F_{jt}\Bigr).
\]
With this method, the complexity is $O(nmr)$ in space and $O(nm^2r)$ in time ($r=\max\{r_x,r_y\})$, which is much more efficient in memory than straightforward implementation.

We introduce two variants of gKDR.
First, as discussed in \cite{Hristache01}, accurate nonparametric estimation for the derivative of regression function with high-dimensional $X$ is not easy in general.  We propose a method for decreasing the dimensionality iteratively in a similar manner to IADE.  Using gKDR, we first find a projection matrix $B_1$ of a larger dimension $d_1$ than the target dimensionality $d$, project data $X_i$ onto the subspace as $Z_i^{(1)} = B_1^T X_i$, and find the projection matrix $B_2$ ($d_1\times d_2$ matrix) for $Z_i^{(1)}$ onto a $d_2$ ($d_2 < d_1$) dimensional subspace.  After repeating this process to the dimensionality $d$, the final result is given by $\hat{B}=B_\ell \cdots B_2 B_1$.  In this way, we can expect the later projector is more accurate by the low dimensionality of the data $Z_i^{(s)}$.  We call this method gKDR-i.

Second, in classification problems, where the $L$ classes are encoded as $L$ different points, the Gram matrix $G_Y$ is of rank $L$ at most.  We can have at most $L$ dimensional subspace by the gKDR method (see \eq{eq:Mn}), which is a strong limitation of gKDR, especially for binary classification.  Note that this problem is shared by many linear dimension reduction methods including CCA and slice-based methods.  To solve this problem, we propose to use the variation of $\widehat{M}_n(x)$ over all points $x=X_i$ instead of the average $\tilde{M}_n$.  We compute the projection matrix $\widehat{B}_i$ from $\widehat{M}_n(X_i)$ at each $i$, take the average of projectors $\widehat{P} = \frac{1}{n} \sum_{i=1}^n \widehat{B}_i\widehat{B}_i^T$, and give the estimator $B$ by the top eigenvectors of $\widehat{P}$.  In practice, eigendecomposition of $\widehat{M}(X_i)$ for all $i$ may not be feasible.  In that case, by partitioning $\{1,\ldots,n\}$ into $T_1,\ldots,T_\ell$, the projection matrices $\widehat{B}_{[a]}$ given by the eigenvectors of $\widehat{M}_{[a]} = \sum_{i\in T_a}\widehat{M}(X_i)$ can be used to define $\widehat{P}=\frac{1}{\ell}\sum_{a=1}^\ell \widehat{B}_{[a]}\widehat{B}_{[a]}^T$.  We call this method gKDR-v.

\if 0
\begin{figure}[tb]
\hrule
\begin{enumerate}
\item Compute Gram matrices: $G_X = (k_\cX(X_i,X_j))$ and $G_Y = (k_\cY(Y_i,Y_j))$,
\item Compute \[
    \tilde{M} = \sum_{i=1}^n\nabla{\bf k}_X(X_i)^T(G_X + n\eps_n I)^{-1}G_Y (G_X + n\eps_n I)^{-1}\nabla{\bf k}_X(X_i),
\]
where $
    \nabla{\bf k}_X(x) =  \bigl(
    \frac{\partial k_\cX(X_1,x)}{\partial x},     \cdots,
    \frac{\partial k_\cX(X_n,x)}{\partial x}\bigr)^T$.
\item Compute the $d$ unit eigenvectors $u_1,\ldots,u_d$ corresponding to the largest $d$ eigenvalues.
\item Output $B = (u_1,\ldots,u_d)$.
\end{enumerate}
\hrule
\vspace*{-2mm}
 \caption{Algorithm of gKDR}
  \label{fig:alg}
\end{figure}
\fi

\subsubsection{Theoretical analysis of gKDR}

Under some conditions, we can obtain the consistency and its rate for $\widehat{M}_n(x)$ and $\tilde{M}_n$.  We assume all the RKHS are separable, and $\|M\|_F$ denotes Frobenius norm of a matrix $M$.
\begin{thm}\label{thm:convergence}
Assume that $\frac{\partial k_\cX(\cdot,x)}{\partial x^a}\in \mathcal{R}(C_{XX}^{\beta+1})$ $(a=1,\ldots,m)$ for some $\beta\geq 0$ and $E[k_\cY(y,Y)|X=\cdot]\in \Hx$ for every $y\in\cY$. Then, for  $\eps_n = n^{-\max\{\frac{1}{3}, \frac{1}{2\beta+2} \}}$, we have
\[
\widehat{M}_n(x) - M(x) =O_p\Bigl( n^{-\min\{\frac{1}{3}, \frac{2\beta+1}{4\beta+4} \}}\Bigr)
\]
for every $x\in\cX$ as $n\to \infty$.
If further $E[\|M(X)\|^2_F]<\infty$ and $\frac{\partial k_\cX(\cdot,x)}{\partial x^a}=C_{XX}^{\beta+1}h_x^a$ with $E\|h_X^a\|_\Hx<\infty$, then $\tilde{M}_n \to E[M(X)]$ in the same order as above.
\end{thm}
The proof is given in Appendix.  Note that, assuming that the eigenvalues of $M(x)$ or $E[M(X)]$ are all distinct, the convergence of matrices implies the convergence of the eigenvectors, thus the estimator of gKDR is consistent to the subspace given by the top eigenvectors of $E[M(X)]$.
%Thus, assuming \eq{eq:edr} and distinct eigenspectrum of $, the estimator $\widehat{B}$ of gKDR is consistent to

\section{Experimental results}

We always use the Gaussian kernel $k(x,\tilde{x})=\exp( -\frac{1}{2\sigma^2}\|x-\tilde{x}\|^2)$  in the kernel method below.

\subsection{Synthesized data}
First we use two types of synthesized data, which have been used in \cite{Hristache01}, to verify the basic performance of gKDR and the two variants.  The data are generated by
\[
(A): \qquad Y = Z\sin(\sqrt{5}Z) + W, \quad Z=\tfrac{1}{\sqrt{5}}(1,2,0,\ldots,0)^T X,
\qquad
\]
\begin{gather}
\hspace*{-4cm}(B): \qquad Y = (Z_1^3 + Z_2)(Z_1-Z_2^3) + W, \nonumber \\ 
Z_1 = \tfrac{1}{\sqrt{2}}(1,1,0,\ldots,0)^T X, \quad Z_2 = \tfrac{1}{\sqrt{2}}(1,-1,0,\ldots,0)^T X, \nonumber
\end{gather}
where $10$-dimensional $X$ is generated by the uniform distribution on $[-1,1]^{10}$ and $W$ is independent Gaussian noise with zero mean and variance $10^{-2}$.  The sample size is $n=100$ and $200$.  The discrepancy between the estimator $B$ and the true projector $B_0$ is measured by $\|B_0 B_0^T (I_m - BB^T)\|_F / d$, where $\|\cdot\|_F$ is the Frobenius norm.  For choosing the parameter $\sigma$ in Gaussian kernel, CV with kNN (${\rm k}=5$) is used with 8 points given by $c\sigma_{med}$ ($0.5\leq c\leq 10$), where $\sigma_{med}$ is the median of pairwise distances of data \cite{Gretton_etal_nips07} (the same strategy is used for CV in all the experiments below).  The regularization parameter is fixed as $\eps_n = 10^{-7}$.

\begin{table}[t]
  \centering
  {\footnotesize
    \begin{tabular}{c|c|c|c|c|c}
       \hline
       % after \\: \hline or \cline{col1-col2} \cline{col3-col4} ...
        & gKDR &  gKDR-i  &  gKDR-v & gKDR+KDR & IADE \cite{Hristache01}\\
       \hline
       (A) $n=100$ & 0.2114 ($0.0636$) & $0.1905$   ($0.0495$) &  $0.2101$  ($0.0704$) & $0.0883$ ($0.1473$) & 0.0903\\
       (A) $n=200$  & 0.1393 ($0.0362$) & $0.1217$  ($0.0352$)  & $0.1356$  ($0.0351$) &  $0.0501$ ($0.0964$) & 0.0537\\
       \hline
       (B) $n=100$ & $0.1500$ ($0.0363$) & $0.1358$ ($0.0347$) &  $0.1630$   ($0.0398$) & $0.1076 $ ($ 0.0967$)  & 0.182\\
       (B) $n=200$  & $0.0755$ ($0.0157$) & $0.0750$ ($0.0153$) & $0.0802$   ($0.0160$)
&  $0.0506$ ($0.0729$)  & 0.0472\\
       \hline
     \end{tabular}
     }
  \caption{Synthesized data.  Mean and standard error (in brackets) over 100 samples.  The mean errors of IADE are taken from \cite{Hristache01}.  }\label{tbl:synthesized}
\end{table}

We compare the results only with IADE, since \cite{Hristache01} reports that the results of IADE are much better than those of SIR and pHd.  From Table \ref{tbl:synthesized}, we see that gKDR, gKDR-i (5 iterations), and gKDR-v show comparable results for data (B), while IADE works better for data (A).  For data (B), when the sample size is 100, the proposed gKDR methods show much better results than IADE.  gKDR and gKDR-v show similar errors, and gKDR-i improves them in all the four cases.
We also use the results of gKDR as the initial state for KDR, which requires non-convex optimization with gradient method.  As we can see from the table, KDR improves the
accuracy significantly, showing results better than or comparable to IADE.  The optimization in KDR, however, sometimes fails to find a good solution, which causes the large variance in the experiments.

\subsection{Real world data}
\label{sec:real_data}

\begin{table}[t]
  \centering
\vspace*{-4mm}
  \begin{tabular}{c|c|c|c}
    \hline
    % after \\: \hline or \cline{col1-col2} \cline{col3-col4} ...
     & Dim. & Train & Test \\
     \hline
    heart-disease & 13 & 149 & 148  \\
    ionoshpere & 34 & 151 & 200 \\
    breast-cancer & 30 & 200 & 369  \\
    \hline
  \end{tabular}
  \caption{Summary of data sets: dimensionality of $X$ and the number of  data}\label{tbl:data_descr}
\end{table}

We first use {\em Wine data}, which is available at the UCI machine learning repository \cite{UCI_repository}, to demonstrate low dimensional visualization.  In this data set, $X$ is a 13 dimensional continuous variable, and $Y$ is the class label representing three classes of wine, which is encoded as $\{(1,0,0),(0,1,0),(0,0,1)\}$.  The sample size is 173.  Two dimensional projections are estimated by gKDR and KDR. For gKDR, the parameter $\sigma$ in Gaussian kernel is chosen by CV with kNN (${\rm k}=5$).  As in Figure \ref{fig:wine}, the results by the KDR and gKDR look similar, while each of the classes by KDR is more condensed.  With Intel (R) Core (TM) i7 960, 3.20GHz, the computational time required for one parameter set was 0.14 sec by gKDR and 4.80 sec by KDR with 50 iterations of line search: gKDR is 30 times faster than KDR for this data set.
As comparison, we show also the results by kernel CCA (KCCA) \cite{Akaho_2001_kcca,KernelICA}.
Since the nonlinear mapping in KCCA easily separates the three classes with small $\sigma$, cross-validation is unstable and inapplicable.  The results given by the three values of $\sigma$ are very different for KCCA.

\begin{figure}
  % Requires \usepackage{graphicx}
  \centering
  \includegraphics[height=3.5cm]{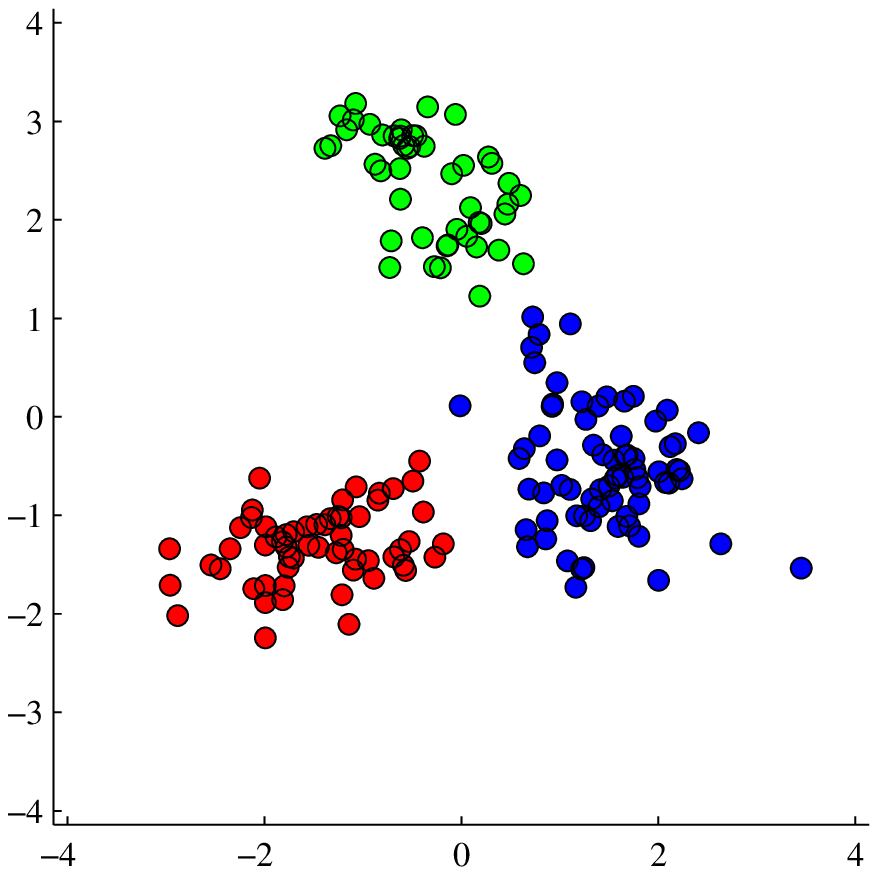}\includegraphics[height=3.5cm]{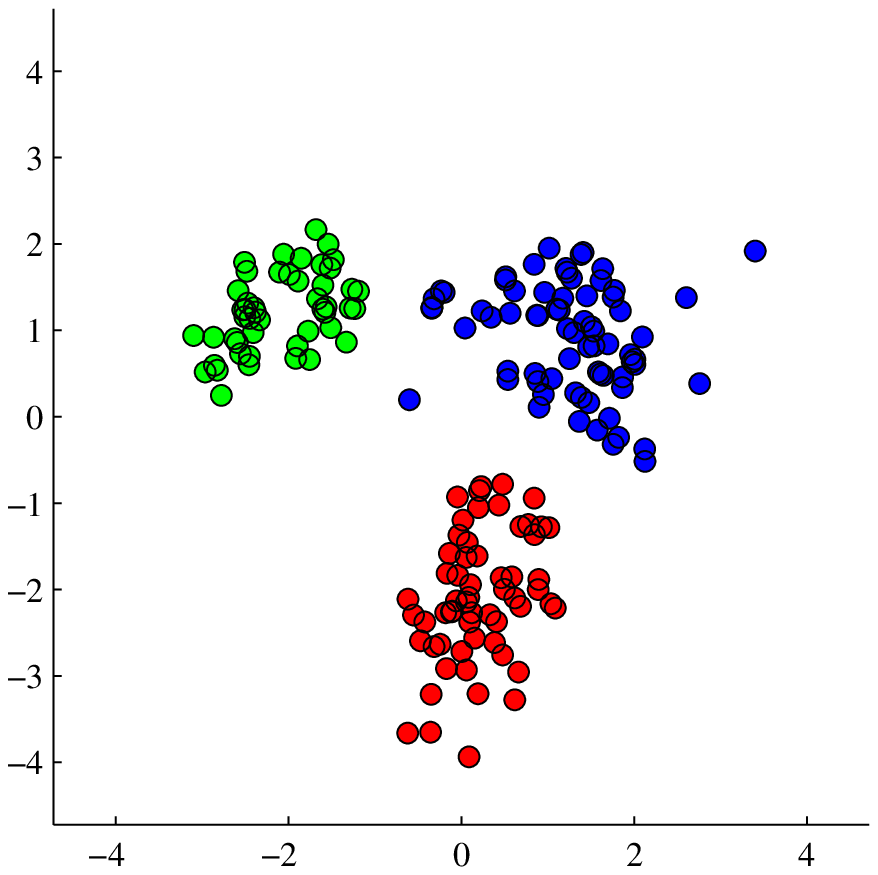}
  \\
  (a) gKDR \hspace*{3cm} (b) KDR \\
  \includegraphics[height=3.5cm]{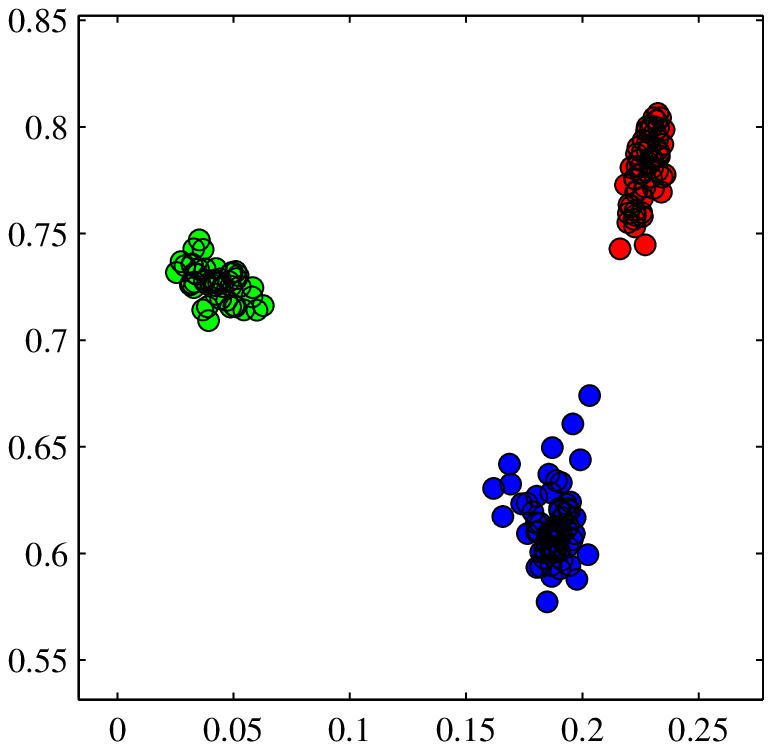} \hspace*{5mm}
  \includegraphics[height=3.5cm]{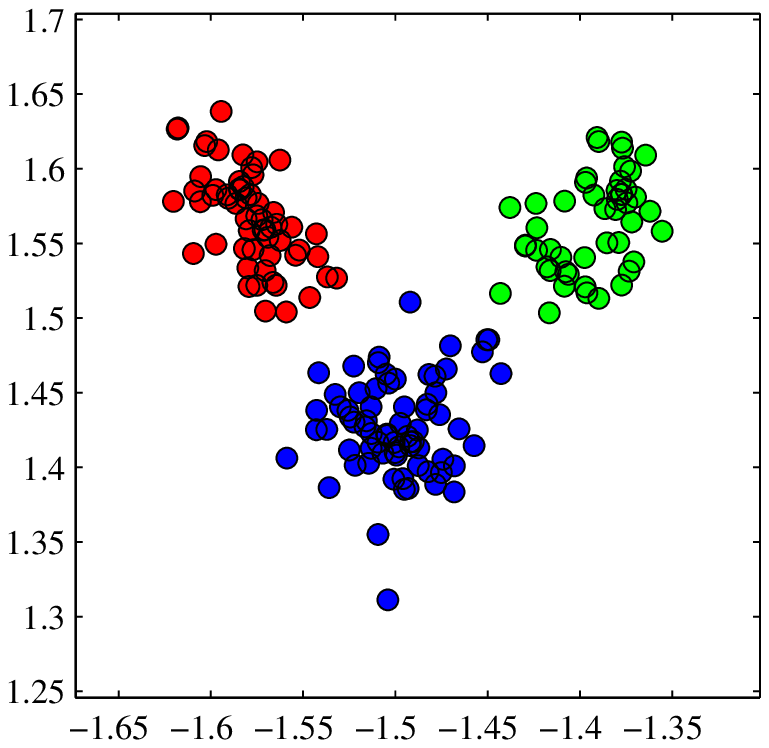}\hspace*{5mm}
  \includegraphics[height=3.5cm]{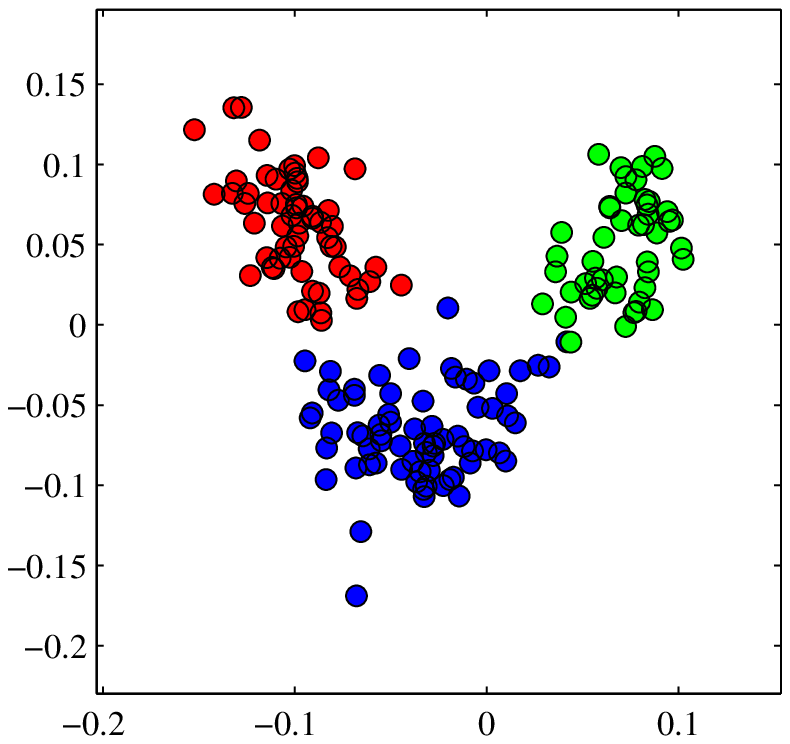} \\
  (c) KCCA ($\sigma=$2*MedD) \hspace*{3mm} (d) KCCA ($\sigma=$10*MedD) \hspace*{3mm} (e) KCCA ($\sigma=$100*MedD)
  \caption{Two dimensional plots of Wine data by gKDR, KDR, and Kernel CCA. MedD means the median of pairwise distances among $X_i$ \cite{Gretton_etal_nips07}. }\label{fig:wine}
\end{figure}

One way of evaluating dimension reduction methods in supervised learning is to consider the classification or regression accuracy after projecting data onto the estimated subspaces.  We next use three data sets for binary classification, {\em heart-disease}, {\em ionoshpere}, and {\em breast-cancer-Wisconsin}, from UCI repository (see Table \ref{tbl:data_descr}), and compare the classification errors with gKDR-v and KDR.

The classification rates with kNN classifiers (${\rm k}=7$) for projected data are shown in Fig.~\ref{fig:gKDRvsKDR}.  We can see that the classification ability of estimated subspaces by gKDR-v is competitive to those given by KDR: slightly worse in Ionosphere, and slightly better in Breast-cancer-Wisconsin.  The computation of gKDR-v for these data sets can be hundreds or thousands times faster than that of KDR. For each parameter set, the computational time of gKDR vs KDR was, in {\em Heart-disease} 0.044 sec / 622 sec ($d=20$), in {\em Ionoshpere} 0.l03 sec / 84.77 sec ($d=20$), and in {\em Breast-cancer-Wisconsin} 0.116 sec /  615 sec ($d=11$).

\begin{figure}
\centering
  % Requires \usepackage{graphicx}
  \includegraphics[height=3.4cm]{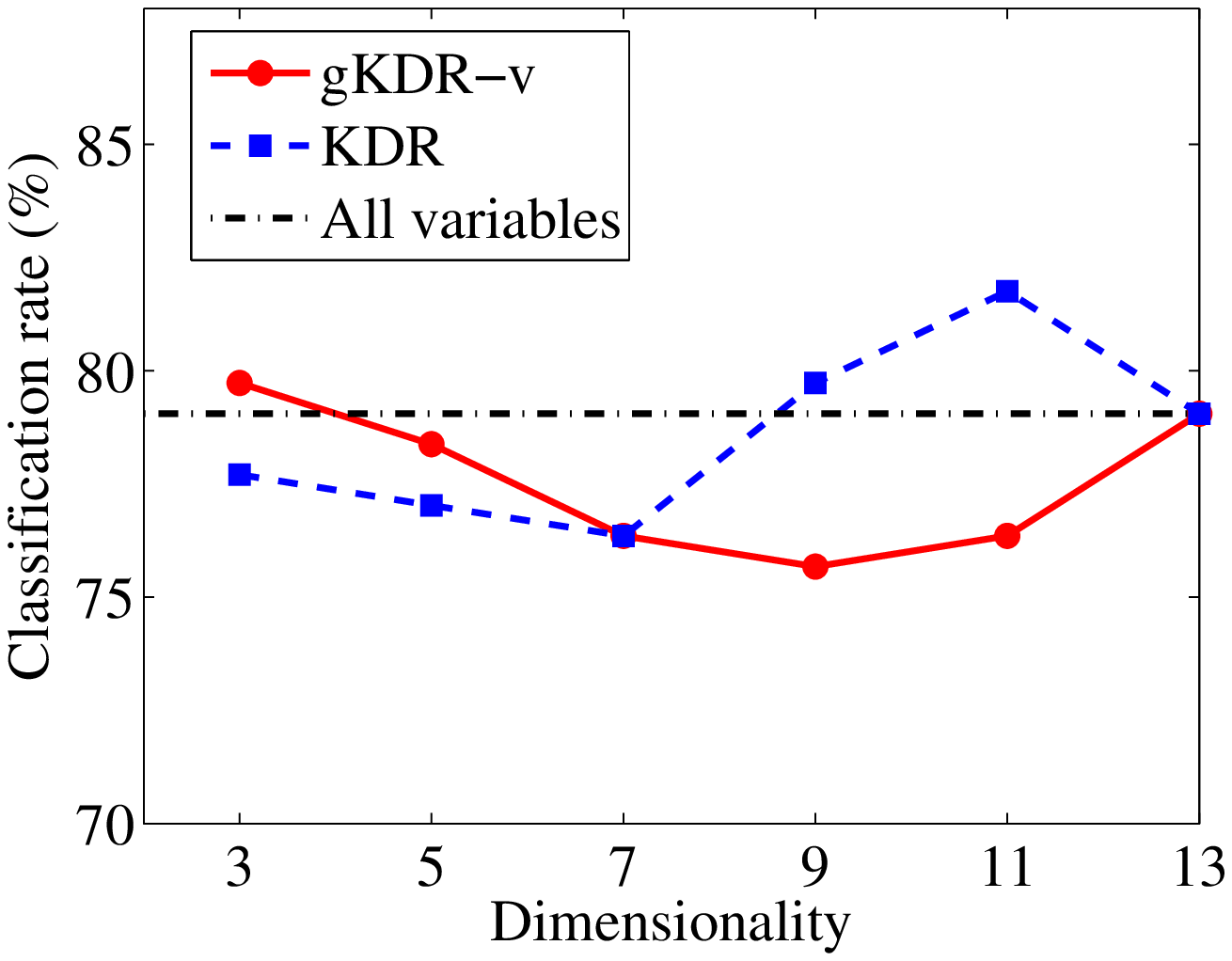}
  \includegraphics[height=3.4cm]{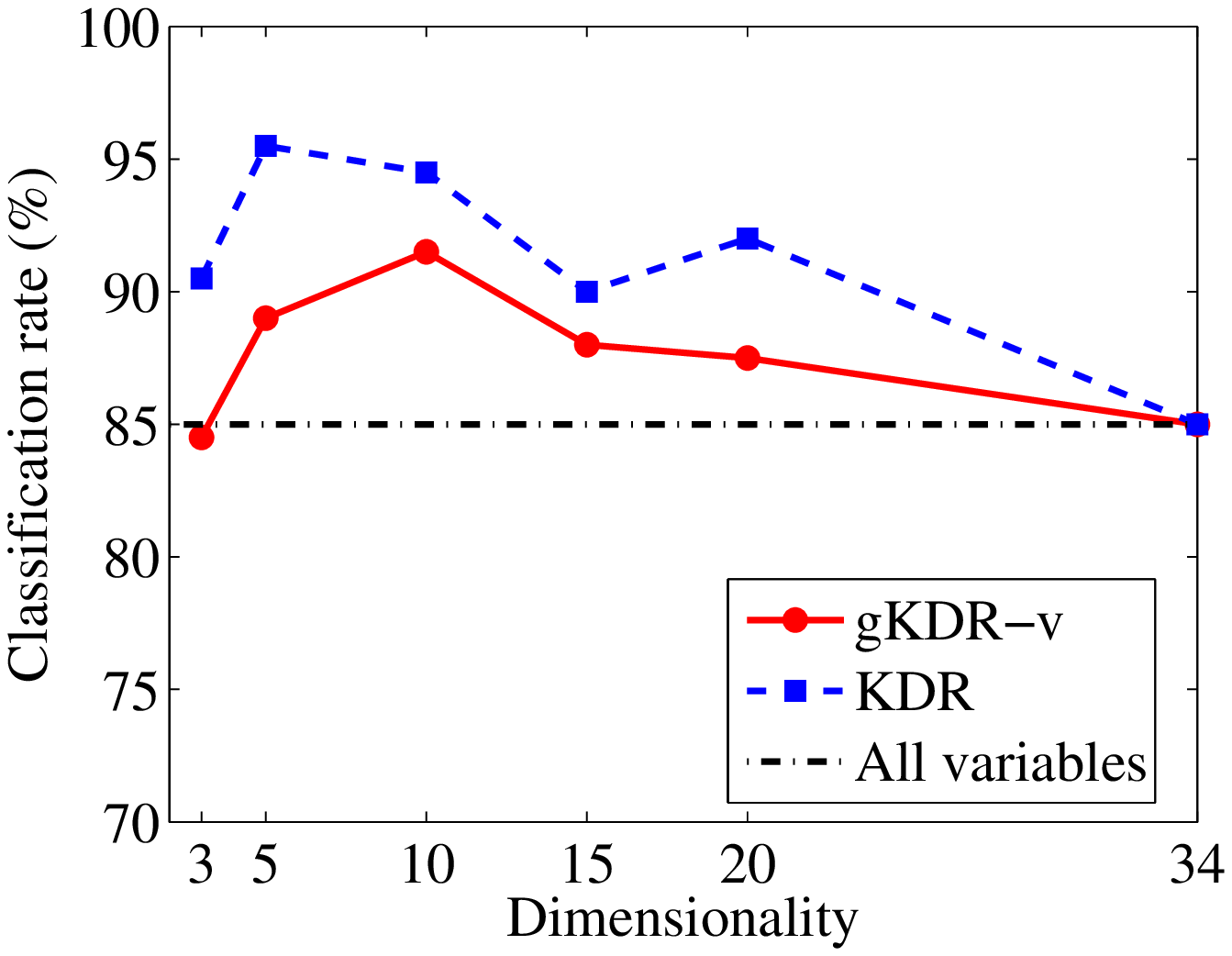}
  \includegraphics[height=3.4cm]{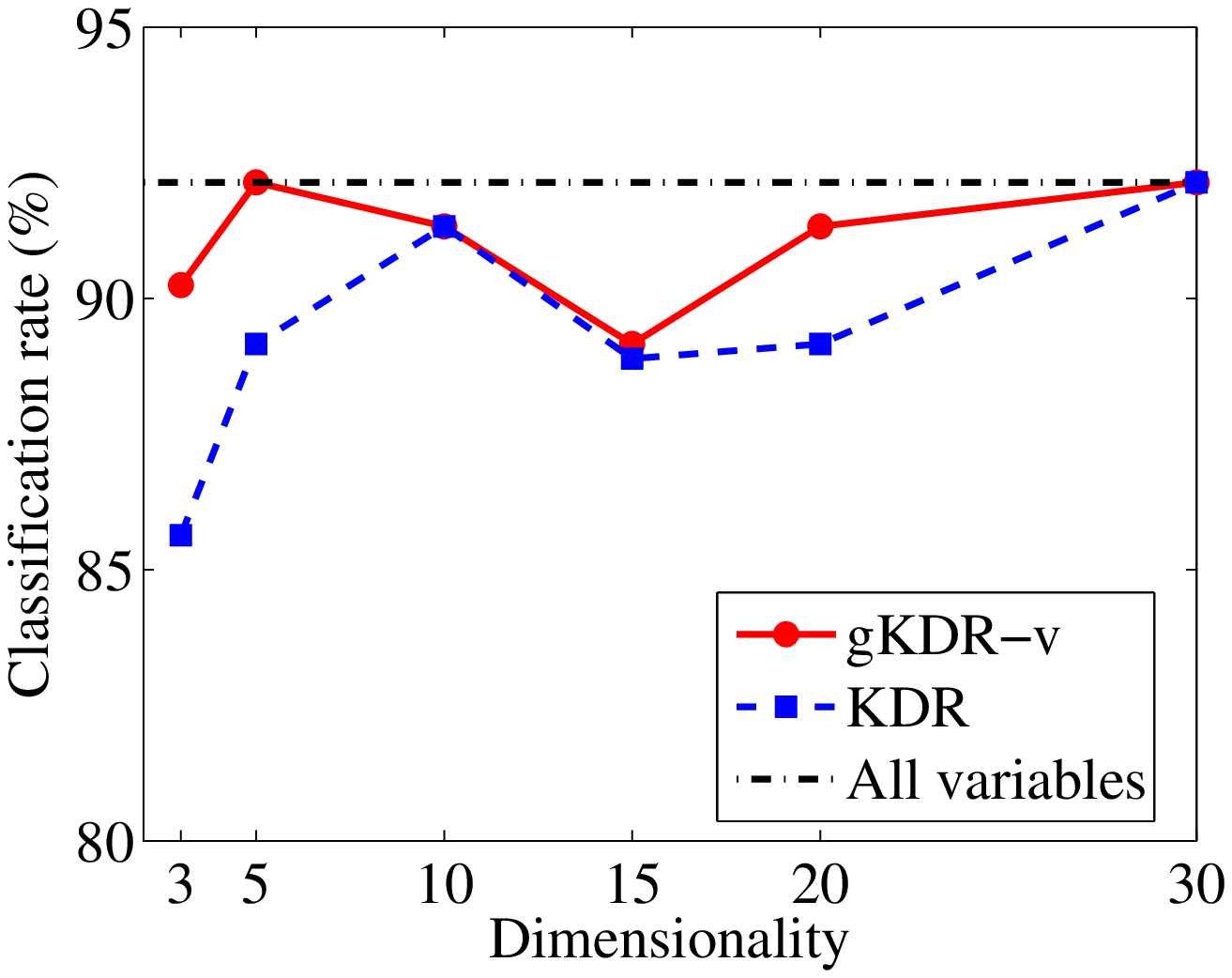}\\
  \hspace*{1cm}(a) Heart Disease \hspace*{1.8cm} (b) Ionoshpere \hspace*{1.7cm} (c) Breast-cancer-Wisconsin
  \caption{Classification accuracy with gKDR-v and KDR for binary classification problems}\label{fig:gKDRvsKDR}
\end{figure}

The next two data sets are larger in the sample size and dimensionality, for which the optimization of KDR is difficult to apply.
The first one is 2007 images of USPS handwritten digit data set used in \cite{Lesong_etal_NIPS2007}, where 256 gray scale pixels are provided as $X$ for each image.  First we make a three dimensional plot for the subset of 500 images with classes ``1" through ``5", as in the similar way to \cite{Wang_etalNIPS2010}.  The result is shown in Fig.~\ref{fig:USPS500}.   We can see, although this is a linear projection, the subspace found by gKDR separates the five classes reasonably well.

\begin{figure}[t]
  % Requires \usepackage{graphicx}
  \centering
  \includegraphics[height=3.5cm]{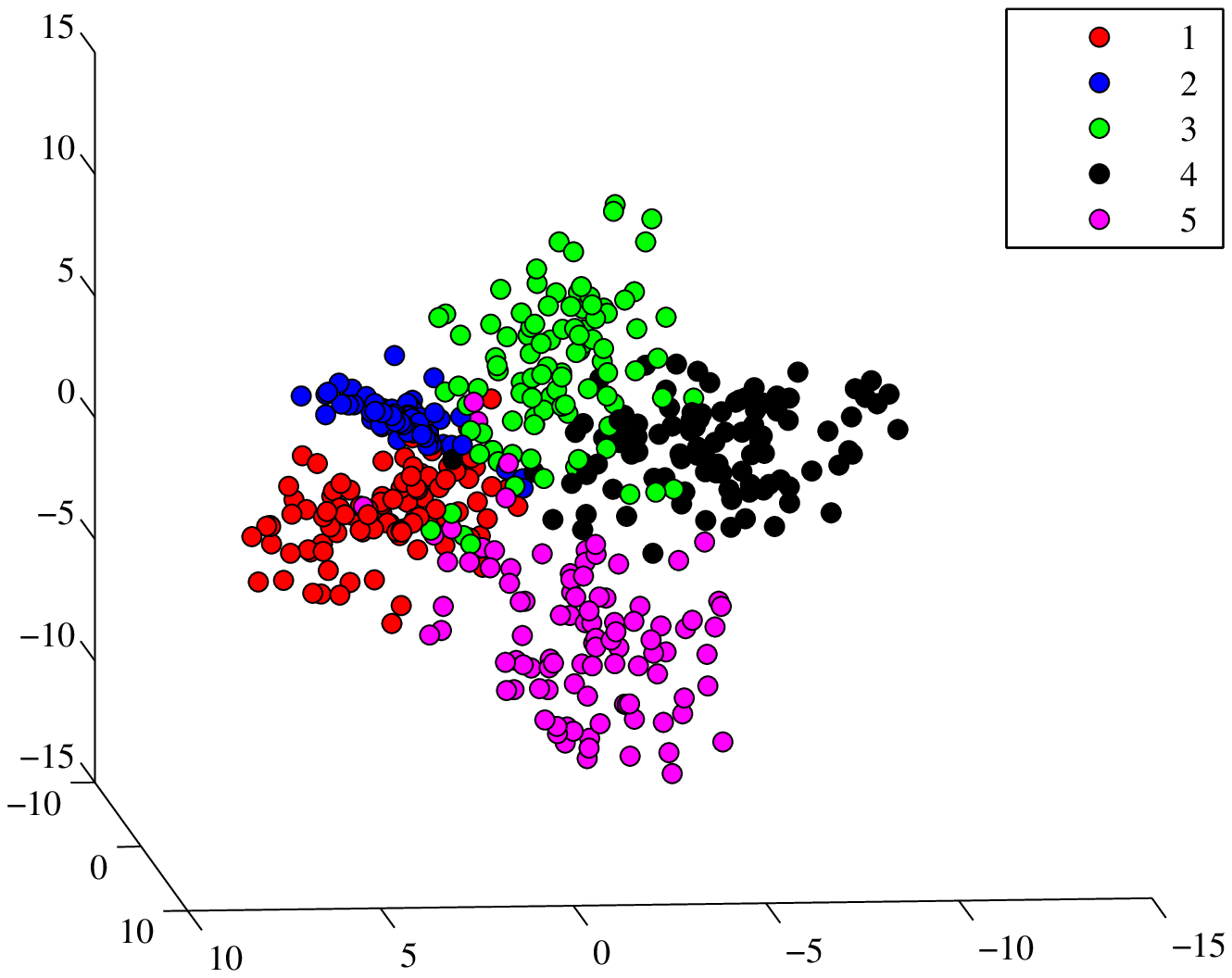}
  \includegraphics[height=3.5cm]{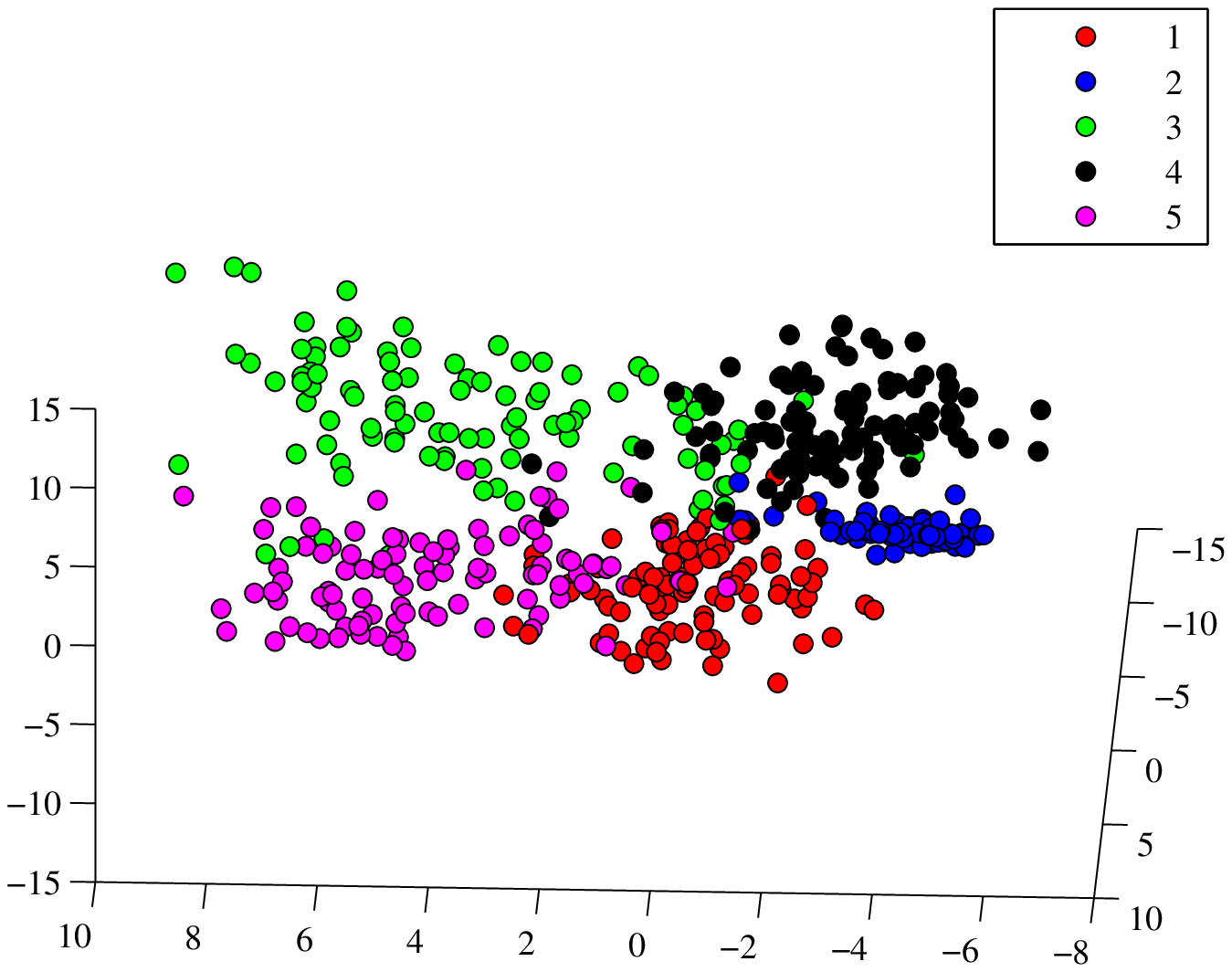}\\
  \caption{Three dimensional plots of USPS data (5 classes) from two different angles.}\label{fig:USPS500}
\end{figure}

We evaluate the classification errors by the simple kNN classifier (${\rm k}=5$) with the data projected onto estimated subspaces, using 1000 images for training and the rest for testing.  We compare gKDR with CCA as a baseline.  Table \ref{tbl:usps} shows that the subspaces found by gKDR (-i,-v) have much better classification ability than those given by CCA.  As in the previous cases, gKDR and gKDR-v show similar errors, and gKDR-i (5 iterations) improves them slightly.

\begin{table}
  \centering
{\small
  \begin{tabular}{c|ccccccc}
    \hline
    % after \\: \hline or \cline{col1-col2} \cline{col3-col4} ...
    Dim. & 3 & 5 & 7 & 9 &  15 & 20 & 25\\
    \hline
    gKDR & 56.82 & 27.96 & 19.00 & 16.66 &   -- & -- & -- \\
    %\hline
    gKDR-i & 39.81 & 26.17 & 18.62 & 15.06 &  -- & -- & -- \\
    %\hline
    gKDR-v & 47.78 & 25.89 & 18.62 &  15.92 & 12.43 & 11.73 & 12.67\\
    %\hline
    CCA & 51.05 & 32.62 & 23.96 & 24.49  &  -- & -- & -- \\
    \hline
  \end{tabular}
  }
  \vspace*{-2mm}
  \caption{USPS2007: classification errors for test data (percentage)}\label{tbl:usps}
\end{table}

The second large data set is ISOLET, taken from UCI repository \cite{UCI_repository}.  The data set provides 617 dimensional continuous features of speech signals for each of 26 alphabets.  In addition to 6238 training data, 1559 test data are separately provided.  We evaluate the classification errors with the kNN classifier (${\rm k}=5$) to see the effectiveness of the estimated subspaces.  Table \ref{tbl:isolet} shows the error rates of classification for the test data after dimension reduction.  To save computational time, we did not use gKDR-i.  From the information on the data at the UCI repository, the best performance with neural networks and C4.5 with ECOC are 3.27\% and 6.61\%, respectively.  In comparison with these results, we can see the simple kNN classification shows competitive performance on the low dimensional subspaces found by gKDR and gKDR-v.

\begin{table}
  \centering
  {\small
  \begin{tabular}{c|cccccccccc}
    \hline
    % after \\: \hline or \cline{col1-col2} \cline{col3-col4} ...
    Dim. & 5 & 10 & 15 & 20 & 25 & 30 & 35 &  40 & 45 & 50\\
    \hline
    gKDR & 30.21 & 13.53 & 7.70 & 4.55 &  4.23 & -- & -- & --& -- & --\\
    %\hline
    gKDR-v & 29.44  & 13.15 & 8.28 & 4.55 & 3.91 & 4.81 & 5.26 & 5.26 & 5.77 & 5.58 \\
    CCA & 22.77 & 15.78 & 8.72 & 6.74 &  7.18 & -- & -- & --& -- & --\\
    \hline
  \end{tabular}
  }
  \vspace*{-2mm}
  \caption{ISOLET: classification errors for test data (percentage)}\label{tbl:isolet}
\end{table}

\section{Concluding remarks}

We have proposed a method for gradient-based kernel dimension reduction and its two variants, which provide general approach for dimension reduction in supervised learning; they have wide applicability with little restriction on the distribution or type of the variables, and the computation is done with simple linear algebra.

As discussed in Sec.~\ref{sec:gKDR_discussion}, gKDR may solve the problem of the existing gradient methods that they do not work if the regression function has the degenerate average derivative.  It is then interesting to make a theoretical question whether gKDR can find the true EDR space.  This is within our future works.

This paper focuses only on the supervised setting, but it may be possible to extend the proposed method to the unsupervised cases in a similar way employed in \cite{Wang_etalNIPS2010}.
Extension to nonlinear feature extraction is also important in some practical problems. As we discuss in Introduction, applying a nonlinear transform will give a straightforward extension.  Another interesting question is how we can ``kernelize" gKDR to replace the linear features to nonlinear ones.  This is not as straightforward as many other kernel methods, since the differentiation with respect to the feature map is involved.  This is also within our interesting future directions.

\appendix

\section{Consistency of the kernel estimator for the regression function}

We discuss the consistency of the estimator $\bigr(\widehat{C}_{XX}+\eps_n I\bigr)^{-1}\hC_{XY} g$ for $E[g(Y)|X=\cdot]$.  While this consistency has been already proved in some literature such as \cite{SmaleZhou2005,SmaleZhou2007,CaponnettoDeVito2007, Bauer_etal2007} in various contexts, we show the proof in our terminology for completeness.

\begin{thm}\label{thm:regression}
Let $g\in\Hy$ and assume that $E[g(Y)|X=\cdot]\in
\mathcal{R}(C_{XX}^\nu)$ for $\nu \geq 0$, where
$\mathcal{R}(C_{XX}^0)$ for $\nu=0$ is interpreted as $\Hx$.  If
$\eps_n\to 0$ ($n\to\infty$),  then
\[
\bigl\| \bigl(\hC_{XX}+\eps_n I\bigr)^{-1}\hC_{XY} g
 - E[g(Y)|X=\cdot]\bigr\|_{\Hx}
\]
is of the order
\[
    \begin{cases}
        O_p(\eps_n^{-1} n^{-1/2}) + O(\eps_n^\nu), & \text{for } 0 \leq \nu < 1, \\
        O_p(\eps_n^{-1} n^{-1/2}) + O(\eps_n), & \text{for }\nu \geq 1.
    \end{cases}
\]
Consequently, if $\eps_n = n^{-\max\{\frac{1}{4},\frac{1}{2\nu+2}\}}$, then the estimator is consistent of the order $O\bigl( n^{-\min\{\frac{1}{4},\frac{\nu}{2\nu+2}\}}\bigr)$.
\end{thm}

\begin{proof}
Take $\eta\in \Hx$ such that $E[g(Y)|X=\cdot] = C_{XX}^\nu \eta$.  From Theorem \ref{thm:cond_mean}, we have
$C_{XY}g = C_{XX}E[g(Y)|X=\cdot] = C_{XX}^{\nu+1}\eta$.

First, we show
\begin{equation}\label{eq:estim_er}
\bigl\| \bigl(\hC_{XX}+\eps_n I\bigr)^{-1}\hC_{XY} g
- (C_{XX}+\eps_n I)^{-1}C_{XY} g \bigr\|_{\Hx}
= O_p(\eps_n^{-1} n^{-1/2}) \qquad (n\to\infty).
\end{equation}
Since $B^{-1}-A^{-1} = B^{-1}(A-B)A^{-1}$ for any invertible operators $A$ and $B$,
the left hand side is upper bounded by
\begin{multline*}
\bigl\| \bigl(\widehat{C}_{XX}^{(n)}+\eps_n I\bigr)^{-1}\bigl(
C_{XX}-\hC_{XX}\bigr)(C_{XX}+\eps_n I)^{-1}C_{XY}g \bigr\|_\Hx
\\
+ \bigl\| \bigl(\widehat{C}_{XX}^{(n)}+\eps_n I\bigr)^{-1}\bigl(
\hC_{XY} -C_{XY}\bigr)g  \bigr\|_{\Hx}.
\end{multline*}
From $C_{XY}g=C_{XX}^{\nu+1}\eta$, we have $\|(C_{XX}+ \eps_n I)^{-1}C_{XY}g\|\leq \|C_{XX}^\nu\eta\|_{\Hx}$.  Combination of this fact with $\|\hC_{XX}- C_{XX}\|=O_p(n^{-1/2})$ proves that the first term is of the order
$O_p(\eps_n^{-1} n^{-1/2})$.  The second term is of the same order
from $\|\hC_{XY}- C_{XY}\|=O_p(n^{-1/2})$, which implies
\eq{eq:estim_er}.

Next, we derive the upper bounds
\begin{equation}\label{eq:approx_er_rate}
\bigl\| \bigl(C_{XX}+\eps_n I\bigr)^{-1}C_{XY} g
- E[g(Y)|X=\cdot]\bigr\|_{\Hx} =
\begin{cases}
        O(\eps_n^\nu), & \text{for } 0 \leq \nu < 1, \\
        O(\eps_n), & \text{for }\nu \geq 1.
    \end{cases}
\end{equation}
It follows from $E[g(Y)|X=\cdot] = C_{XX}^\nu \eta$ and $C_{XY}g = C_{XX}^{\nu+1} \eta$ that
\[
(C_{XX}+\eps_n I)^{-1}C_{XY} g - E[g(Y)|X=\cdot] = (C_{XX}+\eps_n I)^{-1}C_{XX}^{\nu+1} \eta - C_{XX}^\nu \eta.
\]
Let $C_{XX} = \sum_{i} \lambda_i \phi_i \la \phi_i, \cdot\ra$ be the eigendecomposition of $C_{XX}$ such that $\lambda_i>0$ are the eigenvalues and $\phi_i$ are the ohorthonormal eigenvectors.  The eigendespectrum of the operator $(C_{XX}+\eps_n I)^{-1}C_{XX}^{\nu+1} \eta - C_{XX}^\nu$ is then given by
\[
    \frac{\lambda_i^{\nu+1}}{\lambda_i + \eps_n} - \lambda_i^\nu =  \frac{\lambda_i^{\nu}\eps_n}{\lambda_i + \eps_n}\qquad (i=1,2,\ldots).
\]
If $0\leq \nu < 1$, from
$\frac{\lambda^{\nu}\eps_n}{\lambda +\eps_n} = \eps_n^\nu
\frac{\lambda^{\nu}\eps_n^{1-\nu}}{\lambda +\eps_n}\leq
\eps_n^\nu \frac{\eps_n^{1-\nu}}{(\lambda +\eps_n)^{1-\nu}}$ and
$\bigl| \frac{\eps_n^{1-\nu}}{(\lambda +\eps_n)^{1-\nu}}\bigr|\leq
1$ we have
\[
    \| (C_{XX}+\eps_n I)^{-1}C_{XX}^{\nu+1} \eta - C_{XX}^\nu\| \leq \eps_n^\nu.
\]
If $\nu\geq 1$, then $\frac{\lambda^{\nu}\eps_n}{\lambda +\eps_n}
\leq \eps_n\frac{\lambda^{\nu}}{\lambda +\eps_n} \leq \eps_n
\lambda^{\nu - 1}$.  It follows
\[
\| (C_{XX}+\eps_n I)^{-1}C_{XX}^{\nu+1} \eta - C_{XX}^\nu\| \leq \eps_n \|C_{XX}\|^{\nu-1}.
\]
From Eqs. (\ref{eq:estim_er}) and
(\ref{eq:approx_er_rate}), the proof is completed.
\end{proof}

\section{Proof of Theorem \ref{thm:convergence}}

Let $g_a=\frac{\partial k_\cX(\cdot,x)}{\partial x^a}$.  Since
\begin{align*}
M_{ab}(x) & = \Bigl\la \bigl\la E[k_\cY(*,Y)|X=\cdot], g_a\ra_\Hx,  \bigl\la E[k_\cY(*,Y)|X=\cdot], g_b\bigr\ra_\Hx\Bigr\ra_\Hy   \\
& = \bigl\la E[k_\cY(*,Y)|g_a(X)],  E[k_\cY(*,Y)|g_b(X)] \bigr\ra_\Hy
\end{align*}
and
\[
\widehat{M}_{n,ab}(x) = \bigl\la \hC_{YX}\bigl(\hC_{XX}+\eps_n I\bigr)^{-1}
g_a, \hC_{YX}\bigl(\hC_{XX}+\eps_n I\bigr)^{-1} g_b \bigr\ra_{\Hy},
\]
we have
\begin{align*}
& \bigl|\widehat{M}_{n,ab}(x)-M_{ab}(x)\bigr| \\
& \leq \bigl| \bigl\la \hC_{YX}\bigl(\hC_{XX}+\eps_n I\bigr)^{-1}
g_a, \hC_{YX}\bigl(\hC_{XX}+\eps_n I\bigr)^{-1} g_b - E[k_\cY(*,Y)|g_b(X)] \bigr\ra_{\Hy}\bigr|   \\
& + \bigl| \bigl\la
\hC_{YX}\bigl(\hC_{XX}+\eps_n I\bigr)^{-1}
g_a - E[k_\cY(*,Y)|g_a(X)],  E[k_\cY(*,Y)|g_b(X)] \bigr\ra_{\Hy}\bigr| .
\end{align*}
Noting $\eps_n \sqrt{n}\to \infty$ and the expression
\[
\bigl(\hC_{XX}+\eps_n I\bigr)^{-1} = (C_{XX}+\eps_n I)^{-1}\bigl\{ I - \bigl(C_{XX}-\hC_{XX}\bigr)(C_{XX}+\eps_n I)^{-1}\bigr\}^{-1},
\]
Lemma 4 in \cite{SmaleZhou2007} shows that
\[
   \bigl\|  C_{XX}\bigl(\hC_{XX}+\eps_n I\bigr)^{-1} \bigr\| = O_p(1).
\]
From $g_a = C_{XX}^{\beta+1}\eta$ for some $\eta\in\Hx$, we have $\| \hC_{YX}\bigl(\hC_{XX}+\eps_n I\bigr)^{-1}g_a\| =O_p(1)$.  For the proof of the first assertion of Theorem \ref{thm:convergence}, it is then sufficient to prove the following theorem.

\begin{thm}\label{thm:rates}
Assume that $g\in\Hx$ satisfies $\mathcal{R}(C_{XX}^{\beta+1})$ for some $\beta\geq 0$ and that $E[k_\cY(y,Y)|X=\cdot]\in \Hx$ for every $y\in\cY$. Then, for $\eps_n >0$ with $\eps_n = n^{-\max\{\frac{1}{3}, \frac{1}{2(\beta+1)} \}}$, we have
\[
\bigl\| \hC_{YX}\bigl(\hC_{XX}+\eps_n I_n\bigr)^{-1} g - E[k_\cY(\cdot,Y)|g(X)]\bigr\|_\Hy =O_p\Bigl( n^{-\min\{\frac{1}{3}, \frac{2\beta+1}{4\beta+4}\}}\Bigr)
\]
as $n\to \infty$.
\end{thm}
\begin{proof}
It suffices to show
\begin{equation}\label{eq:stat_err}
\bigl\| \hC_{YX}\bigl(\hC_{XX}+\eps_n I\bigr)^{-1} g - C_{YX}\bigl(C_{XX}+\eps_n I\bigr)^{-1} g \bigr\|_{\Hy}^2 = O_p\bigl(\eps_n^{-1/2} n^{-1/2}\bigr)
\end{equation}
and
\begin{equation}\label{eq:bias}
\bigl\| C_{YX}\bigl(C_{XX}+\eps_n I\bigr)^{-1} g - E[k_\cY(\cdot,Y)|g(X)] \bigr\|_{\Hy}^2 = O\bigl(\eps_n^{\min\{1,(2\beta+1)/2\}}\bigr)
\end{equation}
as $n\to\infty$.  In fact, optimizing the rate derives the assertion of the theorem.

Let $g=C_{XX}^{\beta+1}h$, where $h\in\Hx$.
Since $B^{-1}-A^{-1} = B^{-1}(A-B)A^{-1}$ for any invertible operators $A$ and $B$, the left hand side of \eq{eq:stat_err} is upper bounded by
\begin{multline*}
\bigl\| \hC_{YX}\bigl(\hC_{XX}+\eps_n I\bigr)^{-1}(C_{XX}-\hC_{XX})\bigl(C_{XX}+\eps_n I\bigr)^{-1} C_{XX}^{\beta+1} h \bigr\|_\Hy \\
+ \bigl\| (\hC_{YX}-C_{YX})\bigl(C_{XX}+\eps_n I\bigr)^{-1} C_{XX}^{\beta+1} h \bigr\|_\Hy.
\end{multline*}
By the decomposition $\hC_{YX}=\widehat{C}_{YY}^{(n)1/2}\widehat{W}_{YX}\widehat{C}_{XX}^{(n)1/2}$ with $\|\widehat{W}_{YX}\|\leq 1$ (\cite{Baker73}), we have $\|\hC_{YX}\bigl(\hC_{XX}+\eps_n I\bigr)^{-1}\| = O(\eps_n^{-1/2})$.  It is known that $\|C_{XX}-\hC_{XX}\|=O_p(n^{-1/2})$.  %Also, from $(C_{XX}+\eps_n I)^{-1}C_{XX}^{\beta+1}=C_{XX}+\eps_n I)^{\beta}(C_{XX}+\eps_n I)^{-(\beta+1)}C_{XX}^{\beta+1}$ we see $\|(C_{XX}+\eps_n I)^{-1}C_{XX}^{\beta+1}\| \le
From these two fact, we see that the first term is of $O_p(\eps_n^{-1/2}n^{-1/2})$.  Since the second term is of $O_p(n^{-1/2})$, \eq{eq:stat_err} is obtained.

For \eq{eq:bias}, first note that for each $y$
\begin{align*}
E[k_\cY(y,Y)|g(X)] & = \la E[k_\cY(y,Y)|X=\cdot], g\ra = \la E[k_\cY(y,Y)|X=\cdot], C_{XX}^{\beta+1}h\ra \\
& = \la C_{XX}E[k_\cY(y,Y)|X=\cdot], C_{XX}^{\beta}h\ra
=
\la C_{XY}k_\cY(y,\cdot), C_{XX}^{\beta}h\ra \\
& = \la k_\cY(y,\cdot), C_{YX}C_{XX}^{\beta}h\ra
= (C_{YX}C_{XX}^\beta h)(y),
\end{align*}
which means $E[k_\cY(\cdot,Y)|g(X)]=C_{YX}C_{XX}^\beta h$.
Let $C_{YX}=C_{YY}^{1/2}W_{YX}C_{XX}^{1/2}$ be the decomposition with $\|W_{YX}\|\leq 1$.  Then, we have
\begin{multline*}
\bigl\| C_{YX}\bigl(C_{XX}+\eps_n I\bigr)^{-1} g - E[k_\cY(\cdot,Y)|g(X)] \bigr\|_{\Hy} \\
= \|C_{YY}^{1/2} W_{YX} \| \bigl\| C_{XX}^{\beta+3/2}\bigl(C_{XX}+\eps_n I\bigr)^{-1} h - C_{XX}^{\beta+1/2} h \bigr\|_{\Hy}.
\end{multline*}
Let $\{\phi_i\}$ be the unit eigenvectors of $C_{XX}$ such that $C_{XX}f = \sum_{i} \lambda_i \la \phi_i, f\ra$.
Then the eigenspectrum of $C_{XX}^{\beta+3/2}\bigl(C_{XX}+\eps_n I\bigr)^{-1}  - C_{XX}^{\beta+1/2}$ is given by
\[
 - \frac{\eps_n \lambda_i^{(2\beta+1)/2}}{\lambda_i + \eps_n} \qquad (i=1,2,\ldots).
\]
If $0\leq \beta < 1/2$, we have $\frac{ \eps_n\lambda_i^{(2\beta+1)/2} }{ \lambda_i+\eps_n }= \frac{ \lambda_i^{(2\beta+1)/2} }{ (\lambda_i+\eps_n)^{(2\beta+1)/2} } \frac{ \eps_n^{(1 - 2\beta)/2} }{ (\lambda_i+\eps_n)^{(1-2\beta)/2} }\eps_n^{(2\beta+1)/2}\leq \eps_n^{(2\beta+1)/2}$.  If $\beta \geq 1/2$, then $\frac{ \eps_n\lambda_i^{(2\beta+1)/2} }{ \lambda_i+\eps_n } \leq \lambda_i^{\beta-1/2}  \eps_n$.  We have thus \eq{eq:bias}, which  completes the proof of Theorem \ref{thm:rates}

\end{proof}

For the second assertion of Theorem \ref{thm:convergence}, note
\begin{multline*}
\left\| \frac{1}{n} \sum_{i=1}^n \widehat{M}_n(X_i) - E[M(X)]\right\|_{F}
\\ \leq
\left\| \frac{1}{n} \sum_{i=1}^n \widehat{M}_n(X_i) -  \frac{1}{n} \sum_{i=1}^n M(X_i)\right\|_{F}
 + \left\| \frac{1}{n} \sum_{i=1}^n M(X_i) - E[M(X)]\right\|_{F}.
\end{multline*}
The second term in the right hand side is of $O_p(n^{-1/2})$ by the central limit theorem.  By replacing $h$ by $\frac{1}{n}\sum_{i=1}^n h_x^a$ in the proof of Theorem \ref{thm:rates}, the assertion is obtained as a corollary.

\end{document}